\documentclass[12pt]{l4dc2023}
\title[Learning in Time-varying Matching Markets]{Competing Bandits in Time Varying Matching Markets}
\usepackage{times}

\author{%
 \Name{Deepan Muthirayan} \Email{deepan.m@uci.edu}\\
 \addr University of California, Irvine
 \AND
 \Name{Chinmay Maheshwari} \Email{chinmay\_maheshwari@berkeley.edu}\\
 \addr University of California, Berkeley%
 \AND
 \Name{Pramod P. Khargonekar} \Email{pramod.khargonekar@uci.edu}\\
 \addr University of California, Irvine%
 \AND
 \Name{Shankar Sastry} \Email{sastry@coe.berkeley.edu}\\
 \addr University of California, Berkeley%
}
\usepackage{notations}
\usepackage{todonotes}
\begin{document}

\maketitle

\begin{abstract}%

  We study the problem of online learning in two-sided non-stationary matching markets, where the objective is to converge to a stable match. In particular, we consider the setting where one side of the market, the arms, has fixed known set of preferences over the other side, the players. While this problem has been studied when the players have fixed but unknown preferences, in this work we study the problem of how to learn when the preferences of the players are time varying and unknown. Our contribution is a methodology that can handle any type of preference structure and variation scenario. We show that, with the proposed algorithm, each player receives a uniform sub-linear regret of {$\widetilde{\mathcal{O}}(L^{1/2}_TT^{1/2})$} up to the number of changes in the underlying preferences of the agents, $L_T$. Therefore, we show that the optimal rates for single-agent learning can be achieved in spite of the competition up to a difference of a constant factor. We also discuss extensions of this algorithm to the case where the number of changes need not be known a priori.
\end{abstract}


\section{INTRODUCTION}

Matching market design has been influential and instrumental in many markets such as kidney markets, school exchange markets, labor markets etc. Matching markets are essentially two-sided markets where participants on either side of the market have preferences over the other side. The typical goal in a matching market is to match the players on one side to the players on the other side with desirable properties such as stability, maximality, efficiency, and strategy proofness etc. 
Several papers have been published exploring the various properties such as stability (\cite{gale1962college, knuth1997stable}), maximality (\cite{edmonds1965maximum}), strategy proofness (\cite{roth1982economics}), and several applications such as housing allocation \cite{roth1984evolution}, content delivery \cite{maggs2015algorithmic}, kidney exchanges \cite{roth2005pairwise}, control applications such as Transportation \citep[Ride Sharing Platforms]{lowalekar2018online}, Smart Grids \cite[Peer-to-Peer Markets]{muthirayan2020online}, etc. Thus, matching markets have significant relevance and importance.

In spite of the many advances in matching markets, majority of these studies ignore the fact that, in many practical circumstances, the participants of the markets may not know their underlying preferences. In addition, the participants may only be able to observe the outcome after a match and that too with high statistical uncertainty. Therefore, the participants will have to learn their preferences with the added complexity that their exploration can be hampered by the other agents who are competing over the same resources. Thus, in such circumstances, the learning will have to be coordinated in a way that all the participants can uncover their underlying preferences and the market can reach its desired equilibrium. 

\cite{liu2020competing} introduced the competing bandits framework to study the interaction of learning and competition in the context of two-sided online matching markets  settings. The competing bandits framework \cite{liu2020competing} is inspired by many modern matching markets where the markets are two-sided and there are multiple rounds of interactions. In this framework, in each round, the participants submit their preferences and the market is cleared by the submitted preferences. After the market is cleared or the matching is made, the participants get to observe the final outcome or the reward from their matching, which can be prone to statistical uncertainty. 
Thus, the feedback the participants get is like a bandit type feedback. The players need to learn their preferences over the other side of the market from just the bandit feedback, with the added complexity that what they get to explore is also impacted by the other participants. Thus, the competing bandits allows us to study the interaction of learning, where multiple participants are simultaneously learning and competing. The goal is to design an algorithm that can enable the participants to learn their underlying preferences and enable the market to reach a stable match in spite of the complex interplay of competition and learning.

{\bf Contribution}:
We study an extension of the competing bandits problem where the underlying unknown preferences of the participants can be time varying. 
While \cite{liu2020competing} study the competing bandits framework, they assume that the participants' preferences are fixed in time. This may not be realistic, given that the participants' preferences are not likely to be stationary. In this work, we present algorithms and guarantees for the scenario where the participants' preferences can be time variant. While \cite{ghosh2022decentralized} study the same problem for a serial dictatorship preference structure and smoothly varying preferences, {\it the methodology we propose can handle any type of preference structure and any type of variation, and is also very efficient at the same time}. This is our key contribution to this line of work. We show that, with our approach, each agent can incur a regret of {\(\widetilde{O}(L_T^{1/2}T^{1/2})\)}, where \(T\) is the number of rounds and \(L_T\) is the number of changes in the preference ordering over the \(T\) rounds across all the participants, matching the optimal regret rates for single-agent bandit learning \citep{auer2019adaptively}.

\section{RELATED WORKS}
This section discusses the literature on time-varying bandits and learning in matching markets.  

\paragraph{Time Varying Multi-armed Bandits.}
Non-stationary Multi-Armed Bandits (MAB) problem was first studied by \cite{auer2002finite}. They proposed a variant of the EXP3 algorithm called EXP3.S and showed that EXP3.S can achieve Minimax optimal regret $\mathcal{O}(\sqrt{KLT})$ up to logarithmic factors, where $K$ is the number of arms, $L$ is the number of abrupt changes and $T$ is the number of iterations. There were other studies such as \cite{garivier2011upper} and \cite{allesiardo2017non} that demonstrated that a sliding window approach or a restart approach can also achieve the Minimax optimal regret. Recently, \cite{auer2019adaptively} proposed a near-optimal regret algorithm for the same problem that does not need the knowledge of the number of changes. \cite{besbes2019optimal,krishnamurthy2021slowly} studied the non-stationary MAB with slowly changing variations, and showed that a regret of $\widetilde{O}\left(P^{1/3}_T T^{2/3}\right)$ can be achieved, where $P_T$ denotes the total variation. The non-stationary linear bandits problem was studied by \cite{cheung2019learning, russac2019weighted, zhao2020simple}. These works proposed approaches like sliding window least-squares, weighted UCB and restart strategy respectively and showed that a regret of $\widetilde{\mathcal{O}}\left(P^{1/4}_TT^{3/4}\right)$ can be achieved by all these approaches.


\paragraph{Learning in Two-sided Matching Markets.}
Learning in the context of two-sided matching markets is a new and active research area \citep{das2005two,liu2020competing,liu2021bandit,jagadeesan2021learning,maheshwari2022decentralized,basu2021beyond,cen2022regret,dai2021learning}. \cite{das2005two} was the first work which employed multi-armed bandit algorithms to learn preferences in matching markets. However, it was only recently that \cite{liu2020competing} formulated the problem more formally. One can classify the existing literature in this space into two classes. First is the setting where the true underlying preferences of the players over the firms are fixed \citep{liu2020competing,jagadeesan2021learning,cen2022regret,basu2021beyond,sankararaman2021dominate,maheshwari2022decentralized,kong2022thompson}, and second is the setting where the underlying preferences can change with time \citep{min2022learn,ghosh2022decentralized}.
Interestingly, the literature on learning in time-varying matching markets is relatively sparse \citep{min2022learn,ghosh2022decentralized}. \cite{min2022learn} consider the setting of \emph{Markov Matching Market} where the preferences of agents depends on some underlying context which changes based on planner's policy in a Markovian manner. Meanwhile, in \cite{ghosh2022decentralized}, the authors consider the problem of learning in \emph{smoothly varying} matching markets where the underlying preferences of agents are allowed to change by a certain threshold at every time step. Although the authors consider learning in a decentralized context, the preference structure they study is restricted to a serial dictatorship.  

\section{PROBLEM FORMULATION}
Consider a two-sided matching market comprised of  \emph{players} and \emph{arms}, to be referred respectively as the two sides, who gain some positive utility if matched with one another.    
More formally, we denote the set of \(N\) players by $\mathcal{N} = \{p_1, p_2, \dots, p_N\}$ and the set of \(K\) arms by $\mathcal{K} = \{a_1, a_2, \dots, a_K \}$. An important feature of two-sided matching market is that each side of the market has preference over the other side. Preference of any arm \(a_j\in \mathcal{K}\) over players is captured by the utility \(\pi_j\in \mathbb{R}^N_+\), where \(\pi_j(i)\) is the utility derived by arm \(a_j\) on getting matched with player \(p_i\).  Analogously, preference of a player \(p_i\in \mathcal{N}\) over arms is captured by utility \(\mu_{i}\in \mathbb{R}^K_+\), where \(\mu_i(k)\) denotes the utility derived by player \(p_i\) on getting matched with arm \(a_k\).

For the sake of concise notation, if arm \(a_j\) prefers player \(p_i\) over \(p_{i'}\) then we represent it as  
$p_{i'} \succ_{j} p_i$. Similarly, if player \(p_i\) prefers arm \(a_k\) over \(a_{k'}\) then we represent it as \(a_{k}\succ_{i} a_{k'}\).

In what follows, we first describe the framework of two-sided matching market and state relevant background in Section \ref{sec:matchingmarkets}. Following which, in Section \ref{sec:learning-matchingmarkets}, we describe the problem formulation for learning in two-sided time-varying matching markets.

\subsection{Preliminaries on Matching Markets}
\label{sec:matchingmarkets}
In order to formally present the setup in this paper, we recall some relevant concepts on the literature on two-sided matching markets. A matching \(m:\mathcal{N}\rightarrow \mathcal{K}\) is an injective map such that \(m(p)=a\) denotes that player \(p\in\mathcal{N}\) is matched with arm \(a\in\mathcal{K}\). 
\begin{definition}[Blocking Pair]
We say a tuple \((p_i,a_j)\in \mathcal{N}\times \mathcal{K}\) is a blocking pair for a matching \(m\) if player \(p_i\) is matched to an arm \(a_{j'}=m(p_i)\) but 
 $a_j \succ_i a_{j'}$ and $a_{j}$ is either unmatched or $p_i \succ_{j} m^{-1}(a_j)$. We say that the triplet \((p_i,a_j,a_{j'})\) blocks the matching \(m\).
\end{definition}
Having defined the notion of blocking pair, we are now ready to define stable matching. 
\begin{definition}[Stable matching]
A matching \(m\) is called \emph{stable} if there is no blocking pair. Alternatively, a matching is called \emph{unstable} if there exists atleast one blocking pair for it. 
\end{definition}
Gale and Shapley 1962 proposed a polynomial time algorithm -- referred as Deferred-Acceptance (DA) algorithm -- to find a stable matching. 
Without enforcing any specific assumptions on the underlying preference structure, the stable matching is not unique. Therefore, we define the notion of valid partners of a player which captures the set of all arm to which it can match in some stable matching. 
\begin{definition}[Valid partner]
Given the full preference rankings of arms and players, we call arm \(a_j\) to be a valid partner of player \(p_i\) if there exists a stable matching \(\textbf{m}\) such that \(\textbf{m}(p_i)=a_j\).
\end{definition}
We now define two important types of matching which are very crucial for subsequent exposition. 
\begin{definition}[Optimal matching and pessimal matching]
We say a matching \(\optimalMatch\) to be \emph{optimal} matching if every player is matched to its most preferred valid partner. Similarly we say a matching \(\pessimalMatch\) to be \emph{pessimal} matching if every player is matched to its least preferred valid partner.
\end{definition}
In order to find optimal and pessimal matching one can initialize the Deferred-Acceptance algorithm from players side and the arms side respectively. Note that if there is a unique stable matching then the optimal and pessimal matching coincide. Interestingly, in \cite{karpov2019necessary} the authors provide necessary and sufficient conditions on underlying preference structure which ensure unique stable matching.



\subsection{Learning in Time Varying Matching Markets}
\label{sec:learning-matchingmarkets}
While a stable matching can be identified directly by employing Deferred-Acceptance algorithm when the players know their preferences, in many modern matching markets the players can be unaware of their preferences. 
Furthermore, the underlying preferences can be time-varying.
In this paper, we study the problem of bandit learning in time varying matching markets, where the objective is to find a stable match by repeated interactions. As an example, consider the setting of online labor markets  where the two sides of the market are employers (players) and freelancers (arms) respectively. Due to the scale of system, the employers do not know apriori the quality of work by a freelancer and has to repeatedly interact with it. Furthermore, the inherent quality of work by a freelancer could change in a nonstationary manner due to some health reasons, personal problems etc. Consequently the employers have to adapt to this nonstationary change while deciding its preference estimates.


We formulate the repeated market setting as follows. We assume that the preferences of arms are \emph{fixed} and is \emph{common knowledge}. However, the preferences of players are \emph{unknown} and \emph{time-varying}. The players and arms interact with each other for a total of \(T\) rounds, indexed by \(t\). 
At any time \(t\), the true underlying preference
of any player \(p_i\in \mathcal{N}\) over any arm \(a_j\in \mathcal{K}\) is encapsulated by mean reward of its interaction, denoted by \(\mu_{i,t}(j)\), which is unknown. The players repeatedly interact with arms through a platform to learn the underlying utilities. 

At every round $t$, the players submit their \emph{estimate} of their preferences over the arms to a platform based on past rounds of interaction.
The platform then computes a stable matching \(m_t\) based on the submitted preferences and assigns the players to the arms accordingly. Upon being assigned or matched an arm $m_t(i)$, the player $p_i$ pulls the arm $m_t(i)$ and receives a stochastic reward $X_{i,m_t}$ sampled from a 1-{sub}Gaussian distribution with mean $\mu_{i,t}(m_t(i))$.
 The players use the observed reward to update their estimate of preferences over the arms. 

In order to evaluate the performance of any online learning algorithm this setup, we introduce the relevant notion of \emph{regret}. 
Unlike the single player learning setting, in a two-sided matching market, all of the players cannot be assigned their most preferred arm due to the misaligned preference structure of the different players and arms.
With this consideration, \cite{liu2020competing} proposed a regret metric for the stationary preference setting which we extend here to non-stationary setting. This notion of regret, termed as {\it arm-stable regret} with respect to a stable matching \(\stableMatch=(\stableMatch_t)_{t\in[T]}\) corresponding to true underlying preferences, is given by
\beq 
R^i_T(\stableMatch) = \sum_{t=1}^T \mu_{i,t}(\stableMatch_t(i)) - \sum_{t=1}^T \mu_{i,t}(m_t(i)),
\label{eq:regret}
\eeq 
where $\mu_{i,t}(\stableMatch_t(i))$ is the mean reward of the arm that would be assigned to a player $p_i$ in a stable matching outcome if the players know the true underlying preference.
In any two-sided matching market there are two important matchings: \emph{player-optimal} and \emph{player-pessimal} matching. Consequently, there are two associated regret metrics: {\it player-optimal regret}, which is the regret with respect to $\optimalMatch$ optimal matching, and the other notion is the {\it player-pessimal regret}, which is the regret with respect to $\pessimalMatch$ the pessimal match. As pointed out in \cite{liu2020competing}, it is not possible to achieve sub-linear player-optimal regret even in static environment without imposing any assumptions on the underlying preference structure. Therefore,
we adopt the {\it player-pessimal regret} as the metric of performance.


In order to quantify the non-stationarity in an environment, we introduce the notion of time-variation which capture the variability in the underlying true preferences. 
The time-variation in a single player bandit learning is typically specified by a quantity called the total variation \cite{besbes2019optimal} or the total number of changes \cite{auer2019adaptively}. The nature of the interdependence in the matching setting is such the change in other players' preferences can lead to lower reward than pessimal match often even sans changes to its own preferences. Such occurrences are not bounded by the extent of variation and are only dependent on occurrence of changes in the order of player's preferences. Therefore, we adopt the notion of number of changes as the variation measure for the competing bandits setting. Specifically, we define the variation measure as the total number of changes across all the players:
\begin{align}
& L_T = \sum_{t = 2}^T \sum_{i \in \mathcal{N}} \sum_{j \in \mathcal{K}} \mathbb{I}[\mu_{i,t}(j) \neq \mu_{i,t-1}(j)], 
\label{eq:totalvariation}
\end{align} 
where $\mathbb{I}[\cdot]$ is the standard indicator function. Such a definition accounts for all the variations in the market and its impact uniformly across all the players. 
Naturally, if the player's preferences are fixed then \(L_T=0\). 

{In any realistic scenario the preferences does not change arbitrarily over time. Furthermore, it is required to assume certain structure on the underlying preference structure to get any meaningful quantitative guarantees.  Against this backdrop, we make the following assumption on the preferences of players.
\begin{assumption}
We make the following assumptions on the mean reward of arms    
\begin{itemize}
    \item[(i)] The mean reward of arms are bounded. That is \(\mu_{i,t}(k)\leq \bar{\mu}\) for all \(t\in [T], i\in [N], k\in [K]\);
    \item[(ii)] The gap between the arms for any player is always greater than $\Delta$. That is,\newline \(0<\Delta= \min_{t}\)\(\min_{i,k,k'}|\mu_{i,t}(k)-\mu_{i,t}(k')|\)
\end{itemize}
\label{ass:boundedness}
\end{assumption}}
{Some remarks about Assumption \ref{ass:boundedness} are in order.
Assumption \ref{ass:boundedness}-(i) and \ref{ass:boundedness}-(ii) are required to obtain any meaningful regret guarantees for competing bandits in non-stationary matching markets with no structure on preferences. Furthermore, we argue that these are practical assumptions on the reward obtained in any real-world application. Our focus here is on changing preferences on which we do not impose any restriction at all.  We note that we do not require any assumption on the magnitude of instantaneous change on the preferences of any player, as it typically assumed in many prior works on non-stationary bandits \cite{krishnamurthy2021slowly,garivier2011upper}. Furthermore, the change in underlying true preferences of players can occur asynchronously. 

}


\section{ALGORITHM AND RESULTS}

There are three key challenges associated with designing effective algorithm that works in this setting: uncertainty, competition and non-stationarity. Overcoming each of these challenges individually has been studied extensively in the existing literature. However, the key challenge is to develop a provable effective algorithm that can overcome these challenges in a holistic manner. Particularly, any single player cannot explore independently in order learn its preferences because the arm it gets matched with is decided by submitted preferences of all players who are also exploring different arms simultaneously. Furthermore, this challenge gets exacerbated by inherent non-stationarity in the environment as the change in preference of any single player can effect the stable match. Interestingly, we show that the proposed algorithmic design, to be presented, handles this challenge while being able to achieve low player pessimal regret.

\subsection{Algorithmic Description}
Our algorithm is an extension of the bandit learning algorithm proposed by \cite{liu2020competing} for the stationary setting, where players repeatedly interact with the market platform by submitting a preference order over the arms. After receiving the preferences, the platform assigns players to arms as per a stable matching computed using the DA algorithm on the set of preferences submitted by the players. Upon getting matched to an arm, the players receive a reward, which they then use to update their preferences before the next round. The key idea in \cite{liu2020competing} is how the players compute their preferences, which enables them to efficiently explore the arms and let the market converge to the stable outcome in spite of the competition among the players.
Specifically, each player submits the preference order as per the Upper Confidence Bound (UCB) for the mean reward of the arms, computed using the collected reward from previous time steps. 


We extend this algorithm to the time varying setting by adopting a restart strategy. Restart strategy is a widely used strategy for time varying bandit learning \citep{zhao2020simple}. The idea in restart strategy is to restart a base algorithm (such as UCB algorithm) after a certain period $H$. If the number of changes $L_T$ is known, then the restart period can be chosen optimally a priori. We adopt this simple strategy to handle the time variations in the competing bandits case. Specifically, in the algorithm we propose, \emph{the market platform ensures that players restart the local UCB algorithm after every $H$ number of rounds and between two restarts the players run their local UCB algorithms just as in the stationary case \cite{liu2020competing}}. The complete algorithm is shown in Algorithm \ref{alg:bl-cb-tvm}. 

\LinesNumberedHidden{\begin{algorithm2e}[]
\DontPrintSemicolon
\KwInput{Common Restart Period: $H$}
Set $k = 1$

\For{$t = \{1, \dots, T\}$}
{

If \(k=1\), set \(T_{i,t}(j) = 0\) for all \(j\in [K], i\in [N]\)

Players update the UCB, $(u_{i,t}(j))_{j\in K}$ for their respective arms using \eqref{eq:ucb}. 

Players submit their rank ordering $\hat{r}_{i,t}$ computed according to $u_{i,t}(j)$ using \eqref{eq:rankordering}. 

Players are matched to arms as per optimal stable matching $m_t$ computed using $\hat{r}_{i,t}$ and $\pi_j$.

For all \(i\in[N]\), update \(T_{i,t+1}(m_t(i)) = T_{i,t}(m_t(i))+1\)

If \(k\leq H\) update $k = k+1$, else update \(k=1\)
}

\caption{Restart Competing Bandits (RCB) Algorithm}
\label{alg:bl-cb-tvm}
\end{algorithm2e}}

We now introduce the algorithm more formally. Let $T_{i,t}(j)$ denote the number of times player $p_i$ is matched with arm $a_j$ by the platform after the latest restart. After every restart, $T_{i,t}(j)$ is assigned to zero, and $T_{i,t}(j)$ is incremented by one whenever player $p_i$ is matched with arm $a_j$. Let $\hat{\mu}_{i,t}(j)$ denote the mean of the rewards observed whenever player $p_i$ pulls arm $a_j$ after a restart and till $t-1$.
Let \(S_{t}\) denote the latest restart before time step \(t\). Then, $\hat{\mu}_{i,t}(j)$ is given by 
\(
\hat{\mu}_{i,t}(j) = \frac{\sum_{k = S_t}^{t-1} X_{i,m_k} \mathbb{I}[m_k(i) = j]}{T_{i,t-1}(j)}.
\)
Then, for a player $p_i$, the upper confidence bound for arm $a_j$ in the current period is given by
\beq 
u_{i,t}(j) = \left\{ \begin{array}{cc} \infty & T_{i,t}(j) = 0 \\ \hat{\mu}_{i,t}(j) + \sqrt{\frac{3\log(t-S_t)}{2T_{i,t-1}(j)}}& \tn{Otherwise}. \end{array}\right.
\label{eq:ucb}
\eeq 
At the beginning of every round $t$, each player updates its UCB for the arms as in Eq. \eqref{eq:ucb}. The players then compute their respective rank ordering $\hat{r}_{i,t}$ according to the updated UCBs. Specifically, the rank ordering $\hat{r}_{i,t}$ is computed as follows: for any two arms \(a_j, a_{j'}\)
\beq 
\hat{r}_{i,t}(j) > \hat{r}_{i,t}(j') ~ \tn{if} ~ u_{i,t}(j) >  u_{i,t}(j').
\label{eq:rankordering}
\eeq 
The platform receives these rank orderings from every player, computes a stable matching, and assigns the matched arms to the respective players. This completes a round.

Algorithm \ref{alg:bl-cb-tvm} has several important features: first, it is a simple algorithm which us computationally efficient and intuitive. Second, the algorithm does not require storing exact rewards obtained in past rounds at the end of both players and platforms. Third, the preferences of players are updated only using local information which is a crucial aspect to ensure privacy and reliability of the system.
\subsection{Regret Result}

Below, we characterize the regret accrued by a player while using the RCB algorithm (Algorithm \ref{alg:bl-cb-tvm}) under time-varying preferences. 
\begin{theorem}
{Suppose Assumption \ref{ass:boundedness} holds. Under RCB algorithm (Algorithm \ref{alg:bl-cb-tvm}), with the common restart period $ H = L_T^{-1/2}{T}^{1/2}$, the pessimal regret for each player $i$ is given as 
\beq 
R^i_T = \widetilde{\mathcal{O}}\left(L^{1/2}_TT^{1/2}\left(1 + \frac{1}{\Delta^2}\right)\right).\nonumber 
\eeq }
\label{thm:rcb}
\end{theorem}
{Some comments about Theorem \ref{thm:rcb} are in order. First, we note that the constant in the above regret bound can be exponentially large in the size of the market, which is also the case with \cite{liu2020competing} which considered stationary preferences of players.  Second, the achievable regret bounds in a single agent MAB setting is $\widetilde{O}(\sqrt{KL_TT})$ \cite{auer2002finite}. Our regret bounds are similar to the single agent multi-arm bandit setting up to some constant, although under the assumption that the gap between the arm of any two players is greater than $\Delta$ at all the times. As pointed out in \cite{liu2020competing}, the dependence on $\Delta^2$ cannot be improved in general because of the dependence of a player's regret on the gaps of other players in the competing bandit setting. For the same reason, it is not be possible to extend the results to the case where the arm gaps for a player can be arbitrarily close to each other. Third, we note that the regret \(R_T\) can be negative, which is not less desirable as player's can receive better utility than the pessimal matching.

A detailed proof of Theorem 1 can be found in the extended version of this paper \cite{muthirayan2022competing}. However, we provide a sketch of the proof below.

\textbf{Proof Sketch}:
{In the following analysis we only give the key steps. We start by analyzing the agent regret for a player $p_i$ within an interval of length $H$ from one restart to the next. The total regret can then be computed by summing the regret across the intervals of length $H$. Let us denote the start and the end of the $\ell$th such interval by $t^\ell_s$ and $t^\ell_e$. Let us denote the agent regret within the $\ell$th interval by $R^i_\ell$.
Then, the player \(i\)'s regret in interval $\ell$ is 
\( 
R^i_\ell = \sum_{t = t^\ell_s}^{t^\ell_e} \left( \mu_{i,t}(\pessimalMatch_t(i)) - \mu_{i,t}(m_t(i))\right). \nonumber 
\)

The critical part to bounding this is to account for the fact that (i) unlike in a single agent setting, given that all the players are exploring and learning, an unstable match with lower reward than its valid partners can be enforced on a player by the preferences submitted by the other players and (ii) the underlying preferences and consequently the set of stable matching themselves are time varying. While the first part was addressed by \cite{liu2020competing} for the case where the underlying preferences of the players are stationary, our analysis addresses how to account for time variation of the underlying preferences and the further complexity this leads to on how the preferences submitted by other players can force an unstable match. {The key to our analysis is the use of the minimal cover of an unstable match before the occurrence of the first change across the players within the interval.} Below, we present a crucial lemma which provides the pessimal regret for any player between any two restarts. The proof of the lemma is deferred to the appendix. 
\begin{lemma}\label{lem: IntervalLemma}
Suppose Assumption \ref{ass:boundedness} holds. Then, under the RCB algorithm (Algorithm \ref{alg:bl-cb-tvm}), the pessimal regret for a player $i$ between \((\ell-1)th\) restart and \(\ell\)th restart is given by
\beq 
R^i_{\ell} \leq \mathcal{O}(KL_\ell H) + \mathcal{O}\left(K\left(1 + \frac{\log(H)}{\Delta^2}\right) \right). 
\nonumber 
\eeq 
\label{lem:AR}
\end{lemma}

Then, the agent-stable regret for all the rounds can be bound by summing the pessimal regrets for the individual intervals. Hence,
\(
R^i_T \leq \sum_{\ell=1}^{\floor{T/H}} R^i_{\ell} \leq \mathcal{O}(L_T H) + \widetilde{\mathcal{O}}\left( \frac{T}{H\Delta^2} \right). \nonumber 
\)
Substituting for $H=L_T^{-1/2}T^{1/2}$ gives us the final result.}

\begin{remark}[Relation to \cite{liu2020competing}]
The bound we present in Lemma \ref{lem:AR} clearly brings out the relation with respect to \cite{liu2020competing}. The lemma presents the regret incurred by a player within an interval after the restart. The constant accompanying the main term $\widetilde{\mathcal{O}}(1+1/\Delta^2)$ is equivalent to the constant in the time invariant case, which is the sum over the minimal cover induced by blocking pairs that block an unstable match. The difference in our case is that this minimal cover corresponds to the set of blocking pairs before the first change or variation within the interval under consideration, which reduces to the constant in \cite{liu2020competing} in the absence of time variations. The first term in Lemma \ref{lem:AR}, $\mathcal{O}\left(L_\ell H\right)$, reflects the impact on the regret from the time variations. This term obviously vanishes in the absence of time variations.
\end{remark}

\section{EXTENSION TO UNKNOWN TIME VARIATION}
Algorithm \ref{alg:bl-cb-tvm} required knowledge of time variation \(L_T\) in order to compute the optimal restart period \(H\). However, in a lot of real-world applications it is not known a priori. 
One approach to extend the RCB algorithm to the unknown time variation case is the {\it bandits over bandits} approach \citep{zhao2020simple}. The key idea in this approach is to adaptively tune the restart period by exploring the restart periods from an ensemble that ranges from a very low restart period to a very high restart period. Specifically, a {\it meta-bandit algorithm} is used to periodically reset the restart period of the base restart algorithm by treating each period in the ensemble as an arm and employing a bandit algorithm to choose the restart periods while resetting. The bandit algorithm ensures that the exploration of the restart periods is efficient and that the process of exploring the restart period converges to the best restart period.

The same idea can be extended to learning in two-sided matching markets. The platform can employ a bandit meta-algorithm to explore and adaptively choose the best restart period for the base RCB algorithm. To ensure that the ensemble of restart periods is sufficiently rich, the following can be chosen as the ensemble: \(\mathcal{H} = \{H = 2^{j-1}, j \in [1,N] \}, \) where $N = \ceil{1/2\log{T}}+1$. Lets call the restart period defined in Theorem \ref{thm:rcb} by $H^{*}$. Then, $H^{*} = \mathcal{O}(T^{1/2})$ when $L_T = \mathcal{O}(1)$ and $H^{*} = \mathcal{O}(1)$ when $L_T = \mathcal{O}(T)$. Therefore, the ensemble $\mathcal{H}$ as defined contains the $H^{*}$s corresponding to the full range of $L_T$ and is thus sufficiently rich. To explore the restart periods, the meta-algorithm can divide the overall duration $T$ into periods of duration $\Upsilon = \mathcal{O}(T^{1/2})$ and in each period it can explore a specific restart period. The meta-algorithm can employ a standard bandit algorithm such as the EXP3 \citep{auer2002finite} to explore the restart periods efficiently. 

The key issue in the competing bandits setting is how to define the reward for the EXP3 meta-algorithm. We discuss this more elaborately below. Let the matching under RCB with the fixed restart period $H^{*}$ be denoted by $m^{*}_t$. Then, the regret for each player can be split as
\begin{align}
&  R^i_T = \sum_{t=1}^T \mu_{i,t}(\stableMatch_t(i)) - \sum_{t=1}^T \mu_{i,t}(m_t(i)) \nonumber \\
& = \underbrace{\left[\sum_{t=1}^T \mu_{i,t}(\stableMatch_t(i)) -\sum_{t=1}^T \mu_{i,t}(m^{*}_t(i))\right]}_{\tn{Base Regret for $p_i$} } + \underbrace{\sum_{t=1}^T \left[\mu_{i,t}(m^{*}_t(i))\right] - \sum_{t=1}^T \left[\mu_{i,t}(m_t(i))\right]}_{\tn{Meta-regret for $p_i$}} \nonumber 
\end{align}
The first term is the regret incurred when following a fixed restart period. The second term is the difference between the rewards accumulated by following a fixed restart period and the meta-algorithm that adaptively tunes the restart period by exploring from an ensemble of periods. We refer to this second term as the {\it meta-regret}. 

The current proof technique to bound base regret requires that all agents restart simultaneously. Therefore, in this setting we consider that platform coordinates the restarting period uniformly across all players. Consequently, we can only obtain bounds on the joint regret of all players. Towards that goal, we define the reward for the meta-algorithm as the sum total of the accumulated rewards over the period $\Upsilon$ across the players. 
The term $\sum_{t=1}^T \mu_{i,t}(m^{*}_t(i))$ corresponds to the total returns with the restart period $H^{*}$. Since $H^{*} \in \mathcal{H}$, the total return must be less than or equal to the returns corresponding to the best restart period for the players. Therefore, the {\it meta-regret} term must be upper bounded by the actual regret of the meta-algorithm, which can be bounded just as in \cite[Theorem 3]{zhao2020simple} by $\widetilde{O}(T^{3/4})$; see \cite[Theorem 3]{zhao2020simple} for the details. Then, the joint regret when the number of changes are unknown can be bounded as
{\( \sum_{i=1}^\mathcal{N} R^i_T = \widetilde{\mathcal{O}}\left(NL^{1/2}_TT^{1/2}\right) + \widetilde{O}(T^{3/4}). \)

The exploration required for the restart period results in an additional regret on top of the regret in Theorem \ref{thm:rcb}. Therefore, in the unknown $L_T$ case, the achievable total regret is limited by the $\widetilde{O}(T^{3/4})$ term. It is possible that our regret bounds could be improved by a better exploration strategy. In addition, it remains an open problem to bound the individual regret of the players when the number of changes are unknown.

 \section{CONCLUSIONS AND FUTURE WORK}

In this work, we study the framework of competing bandits in time varying matching markets. Specifically, we address the problem of how the players can learn the preferences amidst the competition so that the platform can converge to a stable matching outcome. While this specific challenge has been studied earlier, we address how the players can learn when the players preferences themselves can be time varying. 
For this problem, we propose a {\it Restart Competing Bandits} algorithm that can provably achieve sub-linear regret uniformly across the players up to sub-linear variations in the total number of changes across the players. The key insight is that the non-stationarity does not drastically impact learning in matching markets. Our result is \emph{first-of-its-kind} for a general preference structure. We also discuss extension of our algorithm to the case where the number of changes need not be known a priori and discuss possible open challenges. 
There are several directions for future work. One direction is the meta-algorithm strategy that can achieve the same player pessimal regret for each player as the case where the total number of changes are known. Another major direction is the decentralized setting where there is no central platform to coordinate or communicate with the players. While this problem has been studied for the stationary case, its extension to non-stationary environment is non-trivial. 

\bibliography{Refs}
\appendix

\section{PROOF OF MAIN RESULTS}
\label{sec:mainresult-proof}

Note that proof of Theorem 1, follows immediately from Lemma 1 as stated in Sec 4. Therefore, in this section we provide the proof of Lemma 1.

Before presenting the proof, we introduce some notations which are crucial for the subsequent exposition. Let \(\numChange_\ell\) be the number of time steps in \(\ell\)th interval when there is a change in the underlying preference of any player. For \(k\in [\numChange_\ell]\), let \(t_k^{\ell}\) be the \(k\)th time step when there is change in preference of some player. Let \(t_0^{\ell} = t_s^\ell\).
Before presenting the proof we introduce few definitions about matching market which are crucial in the proof. 

\begin{definition}[Set of all Matchings Blocked by \((p_j,a_k,a_{k'})\)]
For a fixed preference ordering, the set of all matching in which player $p_j$ is matched to $a_{k'}$ and $(p_j,a_k)$ is a blocking pair is denoted by $B_{j,k,k'}$. 
\end{definition}

\begin{definition}[Cover of a Set of Matchings]
Let $Q$ denote be set of triplets $(p_j,a_k,a_{k'})$. For a fixed preference ordering of players and set \(S\) of matchings, we say \(Q\) is a cover of set \(S\), or \(Q\in\mathcal{C}(S)\),   if $\bigcup_{(p_j,a_k,a_{k'}) \in Q} B_{j,k,k'} \supseteq S$.
\end{definition}

\begin{lemma}[Restatement of Lemma \ref{lem: IntervalLemma}]
Suppose Assumption \ref{ass:boundedness} holds. Then, under the RCB algorithm (Algorithm \ref{alg:bl-cb-tvm}), the pessimal regret for a player $i$ between \((\ell-1)th\) restart and \(\ell\)th restart is given by
\beq 
R^i_{\ell} \leq \mathcal{O}(KL_\ell H) + \mathcal{O}\left(K\left(1 + \frac{\log(H)}{\Delta^2} \right)\right). 
\nonumber 
\eeq 
\label{lem:AR}
\end{lemma}

\begin{proof}
We note that the regret in any interval \(\ell\) can be decomposed as follows: 
\begin{align*}
    R_{\ell}^i &=  \mathbb{E}\left[\sum_{t=t_s^\ell}^{t_e^\ell} \left( \mu_{i,t}(\pessimalMatch_t(i)) - \mu_{i,t}(m_t(i)) \right)\right]\\
    &=  \mathbb{E}\left[\sum_{t=t_s^\ell}^{t_e^\ell} \sum_{k=1}^{K}\mathbb{I}(m_t(i)=k) \left( \mu_{i,t}(\pessimalMatch_t(i)) - \mu_{i,t}(k) \right) \right]\\
    &\leq  \mathbb{E}\left[\sum_{t=t_s^\ell}^{t_e^\ell} \sum_{k=1}^{K}\mathbb{I}(m_t(i)=k) | \mu_{i,t}(\pessimalMatch_t(i)) - \mu_{i,t}(k) |\right] \\
    &\leq \mathbb{E}\left[ 2\bar{\mu}\sum_{k=1}^{K}\sum_{t=t_s^\ell}^{t_e^\ell}\mathbb{I}(m_t(i)=k, m_t \text{is unstable})\right] \\ 
    &\leq  \mathbb{E}\left[2\bar{\mu}\sum_{k=1}^{K}\sum_{q=0}^{\Gamma_\ell-1} \sum_{t=t_q^\ell}^{t_{q+1}^{\ell}-1}\mathbb{I}(m_t(i)=k, m_t \text{is unstable})\right] \\ 
    &\leq \mathbb{E}\left[ 2\bar{\mu} \left(\underbrace{ \sum_{k=1}^{K}\sum_{t=t_{0}^\ell}^{t_1^\ell-1}\mathbb{I}(m_t(i)=k, m_t \text{is unstable})}_{\text{Term A}} +  \underbrace{\sum_{k=1}^{K}\sum_{q=1}^{\Gamma_\ell-1} \sum_{t=t_q^\ell}^{t_{q+1}^{\ell}-1}\mathbb{I}(m_t(i)=k, m_t \text{is unstable}) }_{\text{Term B}}\right)\right]
\end{align*}

First we bound Term A. We note that this term can be bounded similar to analysis in \cite{liu2020competing} for the setting where the preferences are stationary. This is because for \(t\in[t_0^\ell,t_1^\ell-1]\)  the preferences do not change and at \(t=t_0^\ell\) the UCB are reintitialized by players as per Algorithm \ref{alg:bl-cb-tvm}. We provide the detailed proof here for the sake of completeness.

For any interval \(\ell\), player \(p_i\) and arm \(a_k\), let \(\mathcal{M}_{ik}^\ell\) denote the unstable matchings where player \(p_i\) is matched to arm \(a_k\) based on the preferences before the first change in the interval \(\ell\). Furthermore, for any matching \(m\) let \(T_m([\tilde{t},\tilde{t}'])\) denote the number of time steps when \(m_t= m\) for \(t\in[\tilde{t},\tilde{t}']\). Note that for any \(k\) and \(Q^\ell \in \mathcal{C}(\mathcal{M}_{ik}^{\ell})\)
\begin{align*}
   \mathbb{E}\left[ \sum_{t=t_{0}^\ell}^{t_1^\ell-1}\mathbb{I}(m_t(i)=k, m_t \text{is unstable})\right] &\leq \mathbb{E}\left[\sum_{m\in\mathcal{M}_{ik}^\ell} T_{m}([t_0^\ell,t_1^{\ell}-1]) \right]\\ 
    &\leq\min_{Q^\ell \in \mathcal{C}(\mathcal{M}_{ik}^{\ell})} \mathbb{E}\left[ \sum_{(p_j,a_s,a_{s'})\in Q^\ell} \sum_{m\in B_{j,s,s'}} T_{m}([t_0^\ell,t_1^{\ell}-1])\right] \\ 
    &\leq\mathbb{E}\left[ \sum_{(p_j,a_s,a_{s'})\in \tilde{Q}^\ell} \sum_{m\in B_{j,s,s'}} T_{m}([t_0^\ell,t_1^{\ell}-1])\right] \\ 
    &\stackrel{(a)}{\leq} \mathbb{E}\left[\sum_{(p_j,a_s,a_{s'})\in \tilde{Q}^\ell} \left(5+\frac{6 \log(H)}{\Delta_{j,s,s'}^2}\right)\right] \\ 
    &\stackrel{(b)}{\leq} \mathbb{E}\left[\sum_{(p_j,a_s,a_{s'})\in \tilde{Q}^\ell} \left(5+\frac{6 \log(H)}{\Delta^2}\right)\right] \\
    &\leq \bar{C}\left(5+\frac{6 \log(H)}{\Delta^2}\right)
 \end{align*} 
where \(\tilde{Q}^\ell \in  \arg\min_{Q^\ell \in \mathcal{C}(\mathcal{M}_{ik}^{\ell})} \mathbb{E}\left[ \sum_{(p_j,a_s,a_{s'})\in Q^\ell} \sum_{m\in B_{j,s,s'}} T_{m}([t_0^\ell,t_1^{\ell}-1])\right] \) and \(\bar{C}=\max_{\ell} |\tilde{Q}^\ell|\). Here inequality (a) follows due to the property of UCB estimates (refer \cite[Proof of Theorem 3]{liu2020competing}) and (b) follows from Assumption \ref{ass:boundedness}-(ii). Therefore, we have 
\begin{align*}
    \mathbb{E}[\text{Term A}] \leq \bar{C}K \left(5+\frac{6 \log(H)}{\Delta^2}\right)
\end{align*}

We next bound Term B. Note that by definition of \(L_\ell\) we have 
\begin{align*}
    \text{Term B} &= \sum_{k=1}^{K}\sum_{q=1}^{\Gamma_\ell-1} \sum_{t=t_q^\ell}^{t_{q+1}^{\ell}-1}\mathbb{I}(m_t(i)=k, m_t \text{is unstable}) \\
    &=  \sum_{k=1}^{K}\sum_{t=t_1^\ell}^{t_e^\ell}\mathbb{I}(m_t(i)=k, m_t \text{is unstable}) \\ 
    &\leq K L_\ell H
\end{align*}

Thus we have 
\begin{align*}
    R_i^\ell \leq 2\bar{\mu}\bar{C}K \left(5+\frac{6 \log(H)}{\Delta^2}\right) + 2\bar{\mu}KL_\ell H
\end{align*}
\end{proof}








\end{document}